\documentclass[letterpaper, 10 pt, conference]{ieeeconf}  

\IEEEoverridecommandlockouts                              

\usepackage{amsmath,amssymb,amsfonts,bm}
\usepackage{algorithm,algpseudocode}
\usepackage{graphicx}
\usepackage{textcomp}
\usepackage{cleveref}
\usepackage{graphicx}
\usepackage{enumerate}
\usepackage{array}
\usepackage{makecell}
\usepackage{cellspace}
\usepackage{braket}
\usepackage{mathtools}
\usepackage{blindtext}
\usepackage{etoolbox}
\usepackage{cite}
\usepackage[outdir=./]{epstopdf}
\usepackage{lipsum}
\usepackage{subfloat}
\usepackage{subcaption}
\usepackage{balance}
\usepackage{float}

{
	\newtheorem{assumption}{Assumption}
	\newtheorem{lemma}{Lemma}
	\newtheorem{theorem}{Theorem}

	\newtheorem{corollary}{Corollary}
	\newtheorem{remark}{Remark}
}
\crefname{figure}{fig.}{figs.}%
\crefname{assumption}{assumption}{assumptions}
\crefname{lemma}{lemma}{lemmata}
\crefname{theorem}{theorem}{theorems}
\crefname{definition}{definition}{definitions}
\crefname{proposition}{proposition}{propositions}

\Crefname{figure}{Fig.}{Figs.}%
\Crefname{Assumption}{Assumption}{Assumptions}
\Crefname{Lemma}{Lemma}{Lemmata}
\Crefname{Theorem}{Theorem}{Theorems}
\Crefname{Definition}{Definition}{Definitions}
\Crefname{Proposition}{Proposition}{Propositions}
\Crefname{Corollary}{Corollary}{Corollaries}
\Crefname{Algorithm}{Algorithm}{Algorithm}

 \makeatletter
 \DeclareRobustCommand{\qed}{%
 	\ifmmode 
 	\else \leavevmode\unskip\penalty9999 \hbox{}\nobreak\hfill
 	\fi
 	\quad\hbox{\qedsymbol}}
 \newcommand{\openbox}{\leavevmode
 	\hbox to.77778em{%
 		\hfil\vrule
 		\vbox to.675em{\hrule width.6em\vfil\hrule}%
 		\vrule\hfil}}
 \newcommand{\qedsymbol}{\openbox}
 \newenvironment{proof}[1][\proofname]{\par
 	\normalfont
 	\topsep6\p@\@plus6\p@ \trivlist
 	\item[\hskip\labelsep\itshape
 	#1:]\ignorespaces
 }{%
 	\qed\endtrivlist
 }
 \newcommand{\proofname}{Proof}
 
 \makeatother
\makeatletter
\newcommand\fs@spaceruled{\def\@fs@cfont{\bfseries}\let\@fs@capt\floatc@ruled
	\def\@fs@pre{\vspace{5\baselineskip}\hrule height.8pt depth0pt \kern2pt}%
	\def\@fs@post{\kern2pt\hrule\relax}%
	\def\@fs@mid{\kern2pt\hrule\kern2pt}%
	\let\@fs@iftopcapt\iftrue}
\makeatother
 
 \makeatletter
 \patchcmd{\fs@ruled}
 {\def\@fs@pre{\vspace{5\baselineskip}\hrule height.8pt depth0pt \kern2pt}}
 {\def\@fs@post{\kern2pt\hrule\relax}}
 {\def\@fs@post{\kern2pt\hrule height 0pt depth .8pt\relax}}
 {}{}
 \makeatother
 \floatstyle{ruled}
 
 \floatname{AlgoEnv}{Algorithm}
 \newfloat{AlgoEnv}{htbp}{loa}[section]
\newcommand\xaug{{\tilde{\bm x}}}

\newcommand{\x}{{\bm x}}

\newcommand{\StatSpAug}{{\tilde{\StatSp}}}
\newcommand{\y}{\bm y}

\newcommand{\xsc}{x}
\newcommand{\gsc}{g}
\newcommand{\D}{{\mathcal{D}}}
\newcommand{\K}{{\bm{K}}}
\newcommand{\f}{{\bm f}}
\newcommand{\fsc}{f}
\newcommand{\g}{{\bm{g}}}

\newcommand{\w}{{\bm{w}}} 
\newcommand{\wsc}{{w}} 

\newcommand{\transp}{^{\text{T}}}

\newcommand{\Id}{\bm{I}}
\newcommand{\NMPC}{T}

\newcommand{\dimx}{{N_x}} 
\newcommand{\dimu}{{N_u}}
\newcommand{\uin}{{\bm u}}
\newcommand{\uinsc}{u}
\newcommand{\xaugsc}{\tilde{x}}
\newcommand{\StatSp}{\mathcal{X}} 

\newcommand{\kernel}{k}
\newcommand{\mean}{m}
\newcommand{\postmean}{\bm{\mu}}
\newcommand{\postmeansc}{\mu}
\newcommand{\step}{t}
\newcommand{\InputSp}{{\mathcal{U}}}

\newcommand{\ParamSp}{{\bm{\varTheta}}}
\newcommand{\nMC}{{M}}
\newcommand{\C}{C}
\newcommand{\Cintegrand}{\sum_{\step=0}^\NMPC c_{\step}}
\newcommand{\covtoyproblem}{0.01}
\newcommand{\nMCtoyprob}{100}

\newcommand{\xdestoyprob}{4\sin( \step/2\pi)}
\newcommand{\nHtoyprob}{{150}}

\newcommand{\gsctoyprxsc}{0.85\sin(12 \xsc_{\step})+\xsc_{\step}^2(\exp(-0.2\xsc_{\step}^2))}

\DeclarePairedDelimiterX\PBasics[1](){ #1}
\DeclarePairedDelimiterX\EBasics[1][]{ #1}

\title{\LARGE \bf
	Anticipating the long-term effect of online learning in control
}

\author{Alexandre Capone
	and Sandra Hirche 
	\thanks{This work was supported by the ERC Starting Grant “Control based on Human Models” under grant agreement no. 337654.}
	\thanks{The authors are with the Department of Electrical and Computer
		Engineering, Technical University of Munich, 80333 Munich,
		Germany (e-mail: alexandre.capone@tum.de; hirche@tum.de). }
}

\begin{document}

	\maketitle
	\thispagestyle{empty}
	\pagestyle{empty}

	\begin{abstract}
		Control schemes that learn using measurement data collected online are increasingly promising for the control of complex and uncertain systems. However, in most approaches of this kind, learning is viewed as a side effect that passively improves control performance, e.g., by updating a model of the system dynamics. Determining how improvements in control performance due to learning can be actively exploited in the control synthesis is still an open research question. In this paper, we present AntLer, a design algorithm for learning-based control laws that anticipates learning, i.e., that takes the impact of future learning in uncertain dynamic settings explicitly into account. AntLer expresses system uncertainty using a non-parametric probabilistic model. Given a cost function that measures control performance, AntLer chooses the control parameters such that the expected cost of the closed-loop system is minimized approximately. We show that AntLer approximates an optimal solution arbitrarily accurately with probability one. Furthermore, we apply AntLer to a nonlinear system, which yields better results compared to the case where learning is not anticipated.
	\end{abstract}
	
	\section{Introduction}
	\label{sect:Introduction}
	Control design often requires an accurate model of the system dynamics. However, obtaining a mathematical model is often prohibitive due to system intricacy or lack of expertise. Moreover, erroneously assuming that a model is correct can lead to poor control performance. These issues have been increasingly addressed by employing online learning-based strategies, i.e., algorithms that employ system measurements collected online to improve control performance. This is typically achieved either by learning a model of the system, e.g., with Bayesian modeling tools~\cite{klenske2016gaussian, murray2002nonlinear, kamthe2017data, koller2018learning, umlauft2019feedback, chowdhary2015bayesian,nguyen2011model}, or by directly learning the optimal control law, e.g., by applying online reinforcement learning~\cite{bakker2006quasionline}.  Despite belonging to the broader category of adaptive control, the intricacy of online learning-based control algorithms often does not allow a formal assessment of the resulting control performance, as opposed to many classical adaptive control strategies~\cite{astrom1994adaptive,Krstic1995}.
	
	Even though online learning-based approaches adapt over time using measurement data, they often include parameters that are \textit{data-independent}, i.e., parameters that are fixed a priori and do not depend on the collected data. Examples include control gains~\cite{kocijan2016modelling,chowdhary2015bayesian,klenske2016gaussian} and safety-relevant parameters~\cite{koller2018learning,berkenkamp2015safe}. Most of these methods choose the data-independent parameters such that system safety and stability is guaranteed after an arbitrary model update~\cite{koller2018learning,berkenkamp2015safe,chowdhary2015bayesian}, while others omit guarantees altogether~\cite{bakker2006quasionline,klenske2016gaussian,kamthe2017data,murray2002nonlinear,kocijan2016modelling,nguyen2011model}. Hence, although learning is an integral part of the control loop, much the same as the control law itself, it only improves the control law in a passive fashion. In other words, the control law is not designed with future learning in mind. This can cause the control to be overly conservative, leading to excessively costly state trajectories.  
	
	Efficiently choosing data-independent parameters in a learning-based setting requires accurately assessing how the control law will perform, which is generally achieved by leveraging any prior knowledge about the system. To this end, we introduce a novel algorithm for optimizing data-independent parameters that quantifies how system uncertainty is expected to be reduced over time due to learning. In other words, the proposed algorithm anticipates the impact that online learning will have on future control performance.
	
	Within the control community, the idea of anticipating and exploiting learning effects in control design has been explored in the form of dual control~\cite{dayan1996exploration,bar1974dual}. So far, dual control has been investigated mostly within the context of structured models with parametric uncertainties, with few exceptions~\cite{klenske2016dual,kral2014gaussian}. However,~\cite{kral2014gaussian} requires the true system to be affine in the control, and both~\cite{kral2014gaussian} and~\cite{klenske2016dual} employ approximations that yield no theoretical guarantees. Hence, developing a general method that provably approximates data-independent parameters arbitrarily accurately remains an open research question.
	
	In this paper, we present AntLer (anticipating learning), a sampling-based algorithm that approximates optimal data-independent parameters of online learning-based control laws in uncertain settings. Our approach accounts for a broad class of model uncertainties by using a probabilistic Gaussian process model. Given a cost function that quantifies control performance over a finite-time horizon, AntLer is able to express the expected cost for an online learning-based control law. Minimizing the resulting expression with respect to the control law's data-independent parameters corresponds to a stochastic optimal control problem, which AntLer solves approximately using sample average approximation. AntLer is applicable to a wide class of dynamical systems that include an additive uncertainty, as well as process noise. We show that, under reasonable assumptions, AntLer approximates the optimal solution arbitrarily accurately given a large enough number of samples. 
	
	The remainder of this paper is organized as follows. \Cref{sect:ProblemSetting} describes the general problem setting and the assumptions used in this paper. In \Cref{sect:GaussianProcesses} the probabilistic approach used to quantify model uncertainty is discussed. \Cref{section:SAA} contains our main result. Therein, we introduce the AntLer algorithm and provide a corresponding theoretical analysis. In \Cref{section:Example} AntLer is applied to a numerical system. We then provide some concluding remarks in \Cref{sect:Conclusion}.
	\paragraph*{Notation}
	Let~$\mathbb{N}$ denote the natural numbers,~$\mathbb{R}$ the real numbers, and~$\mathbb{R_+}$  the non-negative real numbers. We employ bold lowercase and uppercase letters to denote vectors and matrices, respectively. For~$\mu, \sigma \in \mathbb{R_+}$, a normal distribution with mean~$\mu$ and variance~$\sigma^2$ is denoted as~$\mathcal{N}(\mu,\sigma^2)$. For~$d \in \mathbb{N}$, we denote the space of continuously differentiable functions on~$\mathbb{R}^d$ as~$\mathcal{C}^{1}(\mathbb{R}^d)$, and the~$d$-dimensional identity matrix as~$\Id_{d}$. Moreover, for matrices~${\bm{A},\bm{B} \in \mathbb{R}^{d\times d}}$, we use~$[\bm{A},\bm{B}]$ to denote the horizontal concatenation of~${\bm{A}}$ followed by~$\bm{B}$. The entry in the~$i$-th row and~$j$-th column of~${\bm{A}}$ is denoted by~$[\bm{A}]_{ij}$. The symbol~$\cup$ denotes the union of two sets. We use~$\text{E}_{\bm{a}_1, \dotsc, \bm{a}_d}[ \ \cdot \ ]$ to denote the expected value operator with respect to the probability distribution of the random variables~${\bm{a}_1, \dotsc, \bm{a}_d} \in \mathbb{R}^d$.
	\section{Problem Statement}
	\label{sect:ProblemSetting}
	We consider a 
	discrete-time system of the form
	\begin{align}
	\label{eq:SystemDynamics}	
	\begin{split}
	{\x_{\step+1}}  &= \f(\x_{\step},\uin_{\step}) + \g(\x_{\step},\uin_{\step}) + \w_\step \\
	&\eqqcolon \f(\xaug_{\step}) + \g(\xaug_{\step}) + \w_\step
	\end{split}	\end{align}
	where~${\x_{\step} \in \StatSp \subseteq \mathbb{R}^{\dimx}}$ and~${\mathbf{u}_{\step} \in \mathcal{U} \subseteq \mathbb{R}^{\dimu}}$ are the system's state vector and control vector at the~$\step$-th time step, respectively. The initial state~$\x_0 \in \StatSp$ is assumed to be fixed and known. The vector of augmented states~${\xaug_{\step}:= (\x_{\step}, \uin_{\step}) \in \StatSpAug}$, where~${\StatSpAug:=\StatSp \times \mathcal{U}}$, concatenates the state vector~$\x_{\step}$ and the vector of control inputs~$\uin_{\step}$, and is henceforth employed for the sake of simplicity. The system is disturbed by multivariate normally distributed process noise~$\w_\step \sim \mathcal{N}(\bm{0}, \bm{\Sigma}_w^2)$. Here~${\bm{\Sigma}_w = \text{diag}( \sigma_{\wsc_1}, \dotsc, \sigma_{\wsc_\dimx})}$ is a nonnegative diagonal matrix, which we assume to know. The function~${\f \in \mathcal{C}^{1}(\StatSpAug)}$, corresponds to the prior model of the system dynamics, whereas~${\g\in \mathcal{C}^{1}(\StatSpAug)}$ is unknown and is assumed to be drawn from a Gaussian process (GP). This is described thoroughly in \Cref{sect:GaussianProcesses}.

	\begin{remark}
		In this paper, we assume that~$\x_0$ is fixed and known solely to avoid cumbersome notation. The algorithm proposed in this work extends straightforwardly to the more general case where only the probability distribution of~$\x_0$ is known.
	\end{remark}
	
	 \begin{remark}
		This constellation can be assumed for a wide variety of settings. For example, if no prior system knowledge is available, then this is reflected by choosing~$\f(\x_{\step},\uin_{\step}) = \x_{\step}$.
	\end{remark} 

	
	We assume that a parametric online learning-based control law of the form~${\uin: \Gamma \times \ParamSp \times \StatSp \mapsto \InputSp}$ is employed to control~\eqref{eq:SystemDynamics}, where~$\ParamSp$ denotes the space of data-independent control parameters, and~$\Gamma := \{ \left\{\xaug_0, \dotsc, \xaug_{\step} \right\}\in \StatSpAug^\step \ \vert \ \step \in \mathbb{N} \}$ is the set of all finite subsets of~$\StatSpAug$. 
	At every time step, the control law~$\uin(\cdot,\cdot,\cdot)$ takes as arguments the system measurement data~$\D_{\step} = \left\{\xaug_0, \dotsc, \xaug_{\step-1}\right\} \in \Gamma$ collected up to time step~$\step$, the data-independent control parameters~$\bm{\vartheta} \in \ParamSp$, and the current state~$\x_{\step}$. The collected data~$\D_\step$ is employed to update the control law at every time step, e.g., by learning a model of the system. The control parameters~$\bm{\vartheta}$ correspond to the data-independent components of the control law, e.g., multiplicative scalars used to scale confidence regions and thereby guarantee operational safety~\cite{Berkenkamp2016}, or linear feedback gains~\cite{chowdhary2015bayesian}. This formulation encompasses most discrete-time online learning-based control strategies. We henceforth write~${\uin_{\step}(\bm{\vartheta}):=\uin(\D_\step, \bm{\vartheta}, \x_{\step})}$ to denote the online learning-based control law at time step~$t$. 
	
	\begin{remark}
		In order to anticipate the effect of online learning, we aim to predict which data set~$\D_\step$ will be collected over time and how it will affect the overall control performance. As a baseline, we consider the case where predictions are carried out without anticipating learning, which amounts to predicting the closed-loop behavior under the data-independent counterpart~$\uin^0_{\step}(\bm{\vartheta}):=\uin(\D_0, \bm{\vartheta}, \x_{\step})$. Here we assume~$\D_0 \coloneqq \left\{\right\}$ without loss of generality. In \Cref{section:Example}, we compare predictions made with both control laws using a simple example.
	\end{remark}
	
	\begin{remark}
		The method presented in this paper extends straightforwardly to a setting where the system measurements~$\D_\step$ used to update the control law are corrupted by normally distributed observation noise. However, for notational convenience, we focus solely on the case without observation noise.
	\end{remark} 
	
	Our goal is to minimize a finite horizon cost function
	\begin{align}
	\label{eq:CostFun}
	\C (\bm{\vartheta}) \coloneqq \text{E}_{\xaug_1, \dotsc, \xaug_{\NMPC}}\left[{\sum \limits_{{\step}=0}^{\NMPC} c_\step(\x_{\step},\uin_{\step}(\bm{\vartheta}))}\right]
	\end{align}
	over the data-independent control parameters~$\bm{\vartheta}$, where~${c_t: \StatSpAug \mapsto \mathbb{R}_{+}}$ are continuously differentiable functions that express the immediate cost. The probability distribution of~$\xaug_1, \dotsc, \xaug_{\NMPC}$ captures both the effect of process noise~$\w_\step$, as well as the model uncertainty~$\g(\cdot)$. This is discussed in \Cref{sect:GaussianProcesses}. We denote the minimal value of~\eqref{eq:CostFun} as~$\C^* := \min_{\bm{\vartheta} \in \ParamSp} \C(\bm{\vartheta})$ and the corresponding set of minimizing parameters as~${\ParamSp^* := \left\{ \bm{\vartheta} \in \bm{\varTheta} \ \vert \  \C(\bm{\vartheta}) = \C^*\right\}}$. 
	
	If no assumptions about the online learning-based control law~$\uin_{\step}(\cdot)$ are made, then it is generally impossible to reliably predict the closed-loop behavior of~\eqref{eq:SystemDynamics}. Hence, we need to impose some restrictions on the type of control law considered. 
	 
	\begin{assumption}
		\label{assumption:GrowthofCosts}
		There exists a compact subset~$\tilde{\ParamSp} \subseteq \ParamSp$, such that~${\sum_{{\step}=0}^{\NMPC} c_\step(\x_{\step},\uin_{\step}(\bm{\vartheta}))} > \C^*$ holds for all~${\bm{\vartheta} \in \ParamSp \setminus \tilde{\ParamSp}}$ and arbitrary~$\x_0, \dotsc, \x_{\NMPC} \in \StatSp$. 
	\end{assumption}

	\Cref{assumption:GrowthofCosts} is less restrictive than assuming that~$\ParamSp$ is compact, which is often the case in learning-based applications, e.g., in settings where safety-relevant constraints are an issue~\cite{berkenkamp2016safe,neumann2019data}. 
	 Furthermore, \Cref{assumption:GrowthofCosts} does not impose strong limitations in practice, as~$\tilde{\ParamSp}$ may still be very large.
	
%
	
	  In order to be able to find a minimizer~$\bm{\vartheta}^* \in \ParamSp^*$ of~\eqref{eq:CostFun}, we additionally require the control law~$\uin(\cdot,\cdot,\cdot)$ to satisfy some regularity conditions. In this paper, we restrict the control law to the broad and practically relevant class of continuously differentiable functions, as described in the following.
	\begin{assumption}
		\label{assumption:uiscontinuous}
		The control law~$\uin(\D_\step, \bm{\vartheta}, \x_{\step})$ is continuously differentiable with respect to its arguments, where continuous differentiability with respect to the data is defined as follows. For every fixed~$\D \in \Gamma$,~$\x_{\step} \in \StatSp$ and~$\bm{\vartheta} \in \ParamSp$, the function
		\begin{align*}
		\uin_{\D, \bm{\vartheta}, \x_{\step}}(\xaug)\coloneqq \uin(\D \cup \xaug, \bm{\vartheta}, \x_{\step})
		\end{align*}
		is continuously differentiable with respect to~$\xaug$ for all~$\xaug\in \StatSpAug$. 
	\end{assumption}

	Many commonplace control laws are continuously differentiable with respect to the state~$\x_{\step}$ and parameters~$\bm{\vartheta} \in \ParamSp$, e.g., linear feedback gains and neural networks. Furthermore, control update rules are often continuously differentiable with respect to the data, e.g., if a model of the system is learned online \cite{umlauft2019feedback}.
	
	\section{Probabilistic system model}
	\label{sect:GaussianProcesses}
	In this section, we provide a brief introduction to GPs, and describe how we use them to capture model uncertainty and predict control performance.
	\subsection{Predictions using Gaussian processes}
	In order to assess how the learning-based control law will perform in an uncertain environment, we require a probabilistic model that expresses model uncertainty given prior system measurements. To this end, we model~\eqref{eq:SystemDynamics} using a \textit{Gaussian process} (GP), a probabilistic modeling tool that captures model uncertainty. We opt to employ GPs in this work because they often exhibit good generalization behavior in practice. However, we note that other probabilistic modeling frameworks can be employed, e.g., Bayesian neural networks. 
	
	We introduce GPs for the case where the state is a scalar, i.e.,~$\dimx=1$, and then explain how one-dimensional GPs are extended to the multivariate case. A GP is a collection of dependent random variables, for which any finite subset is jointly normally distributed~\cite{Rasmussen2006}. It is fully specified by a mean function~$\mean: \StatSpAug \mapsto \mathbb{R}$ and a positive definite kernel function~$\kernel: \StatSpAug \times \StatSpAug \mapsto \mathbb{R}$. In this paper, since our prior knowledge is captured by~$\fsc(\cdot)$, we set~$\mean \equiv 0$ without loss of generality~\cite{Rasmussen2006}. The kernel~$\kernel(\cdot,\cdot)$ is a similarity measure for evaluations of~$\gsc(\cdot)$, and encodes function properties such as smoothness and periodicity. Throughout this paper, we assume that the kernel~$k(\cdot,\cdot)$ is continuously differentiable, which reflects the assumption that~$\gsc(\cdot)$ is continuously differentiable~\cite{Rasmussen2006}. Given~$\mean(\cdot)$ and~$\kernel(\cdot,\cdot)$, we denote a GP by~$\mathcal{GP}(\mean,\kernel)$. By modeling an unknown function~$\gsc(\cdot)$ with a GP, we implicitly assume that any finite set of function evaluations~$\y_{\D_{\step}}\coloneqq (\gsc(\xaug_0), \dotsc, \gsc(\xaug_{\step-1}))$ at arbitrary points~$\D_{\step} \coloneqq \{\xaug_0, \dotsc, \xaug_{\step-1} \}$ is jointly normally distributed,
	\begin{align}
	\label{eq:GPDistribution}
		\y_{\D_{\step}} \sim  \mathcal{N}\left(\bm{0},\bm{K}_{D_{\step}}\right), 
	\end{align}
	where the entries of the covariance matrix~$\K_{\D_{\step}}$ are given by~$[\K_{\D_{\step}}]_{ij}=\kernel(\xaug_{i-1},\xaug_{j-1})$,~$i,j =1,\dotsc,\step$.

	Using~\eqref{eq:GPDistribution}, we are able to condition the GP on any measurements taken prior to the control design. In the following, for the sake of notational simplicity, we assume that no prior measurement data is available, and describe how to recursively draw and condition the GP on samples. However, conditioning the GP on system measurement data is identical to conditioning on samples up to an additive term that represents noise covariance~\cite{Rasmussen2006}.
	
	In order to predict the control performance of~$\uin_{\step}(\cdot)$, we aim to draw sample trajectories that satisfy~\eqref{eq:GPDistribution}. We henceforth distinguish sample evaluations of the GP model, which are drawn using~\eqref{eq:GPDistribution}, from evaluations of the true system~\eqref{eq:SystemDynamics} by denoting samples using the superscript~$s$. A sample system trajectory is computed by sequentially sampling from the one-step prediction of the unknown dynamics at time step~$\step$
	\begin{align}
	\label{eq:OneStepPred}
	\begin{split}
	&\gsc^s(\xaug^{s}_{\step}) 
	\sim \mathcal{N} \left( \postmeansc^{s}_{\step} \left(\xaug^{s}_{\step}\right) ,\left(\sigma^{s}_{\step}\left(\xaug^{s}_{\step}\right) \right)^2  \right),
	\end{split}
	\end{align}
	and subsequently computing the next sample state
		\begin{align}
	\label{eq:NextSampleState}
	\begin{split}
	\xaug^{s}_{\step+1} 
	=&\left(\xsc^{s}_{\step+1} , \uinsc_{\step}^s(\bm{\vartheta} )\right) \coloneqq \left(\xsc^{s}_{\step+1} ,\uinsc(\D_{\step}^s,\bm{\vartheta},\xsc^{s}_{\step+1} )\right),\\
		\xsc^{s}_{\step+1} 
	=&\fsc(\xaug^{s}_{\step}) + \gsc^s_{\step}(\xaug^{s}_{\step}) + \wsc_{\step}^s ,
	\end{split}
	\end{align}
where~${\D_{\step}^s\coloneqq\{\xaug^{s}_{0},\dotsc,\xaug^{s}_{\step-1}\}}$ and~$\wsc_{\step}^s\sim \mathcal{N}(0,\sigma_{\wsc})$. Here ${\xaug^{s}_{0} \coloneqq (\x_0,\uin_0(\bm{\vartheta}))}$ is introduced for simplicity of exposition. The mean and variance of~\eqref{eq:OneStepPred} are computed using
\begin{align}
\label{eq:GPMean}
\begin{split}
\postmeansc^{s}_{\step} \left(\xaug^{s}_{\step}\right)\coloneqq &\postmeansc \left(\xaug^{s}_{\step}\vert \D^{s}_{\step},\y_{\D^{s}_{\step}}\right) = \bm{\kernel}^{\text{T}}\left(\xaug^{s}_{\step}\right) \K_{\D^{s}_{\step}}^{-1} \y_{\D^{s}_{\step}}  \transp
\end{split}\\
\label{eq:GPVar}
\begin{split}	
\left(\sigma^{s}_{\step} \left(\xaug^{s}_{\step}\right)\right)^2 \coloneqq & \sigma^2\left(\xaug^{s}_{\step}\vert \D^{s}_{\step},\y_{\D^{s}_{\step}} \right)  \\ =&\kernel\left(\xaug^{s}_{\step},\xaug^{s}_{\step}\right) - \bm{\kernel}\transp\left(\xaug^{s}_\step\right)\K_{\D^{s}_{\step}}^{-1} \bm{\kernel}\left(\xaug^{s}_{\step}\right),
\end{split}
\end{align}
respectively. Here the vector
\begin{align}
\label{eq:GP_y}
\y_{\D_{\step}^s}\coloneqq (\gsc^s(\xaug^s_0), \dotsc, \gsc^s(\xaug^s_{\step-1}))
	\end{align} 
concatenates previously drawn sample states~$\xaug^{s}_i \in \D^{s}_{\step}$, and
\begin{align}
\label{eq:GPKernelVector}
	\bm{\kernel}\left(\xaug^{s}_{\step}\right) = \Big( \kernel(\xaug^{s}_0, \xaug^{s}_{\step}), \ \dotsc, \ \kernel(\xaug^{s}_{\step-1}, \xaug^{s}_{\step}) \Big)\transp
\end{align}
consists of kernel evaluations at~$\xaug^{s}_{\step}$ and~$\xaug^{s}_i\in \D^{s}_{\step}$.
\begin{remark}
	Here we abuse notation slightly by employing~$\gsc^s(\cdot)$ to refer to a function sampled from the GP. As can be seen from~\eqref{eq:GPMean}-\eqref{eq:GPKernelVector},~$\gsc^s(\cdot)$ depends on previously sampled function evaluations. In fact, a sample function evaluation is computed as
	 \begin{align}
	\label{eq:SampleRewritten}
	\begin{split}
	&\gsc^{s}(\xsc_{\step+1})   =  \postmeansc^{s}_{\step} \left(\xaug^{s}_{\step}\right) + \sigma^{s}_{\step}\left(\xaug^{s}_{\step}\right) \zeta^s,
	\end{split}
	\end{align}
	where~$\zeta^s\in \mathcal{N}(0,1)$. In~\Cref{section:SAA}, we use rigorous notation by referring to sample function evaluations as in~\eqref{eq:SampleRewritten}.
\end{remark}
	\begin{remark}
	It is necessary that the GP samples~$\gsc^s(\xaug^{s}_{\step})$ and process noise samples~$\sigma_{\wsc}^s$ be drawn separately in order for the vector~$\y_{\D^s_{\step}}$ to be uniquely defined. This in turn guarantees that the sampled function~$\gsc^s(\cdot)$ exhibits deterministic behavior at points where samples were previously drawn~\cite{Rasmussen2006}. We require this to reflect the fact that~$\gsc(\cdot)$ is unknown but deterministic. Hence, we draw sample trajectories that satisfy~\eqref{eq:OneStepPred} and~\eqref{eq:NextSampleState} as
	\begin{align}
	\label{eq:OneStepPredRewritten}
	\begin{split}
	&\xsc^{s}_{\step+1}   = \fsc(\xaug^{s}_{\step}) + \postmeansc^{s}_{\step} \left(\xaug^{s}_{\step}\right) + \sigma^{s}_{\step}\left(\xaug^{s}_{\step}\right) \zeta_1^s + \sigma_{\wsc} \zeta_2^s,
	\end{split}
	\end{align}
	where~$\zeta^s_1, \zeta_2^s \sim \mathcal{N}(0,1)$ are sampled separately.
\end{remark}

\begin{remark}
	Typically, multiple samples can be drawn from the same GP simultaneously~\cite{Rasmussen2006}. However, since we are interested in samples that satisfy the system dynamics, we need to draw a sample and compute the resulting state sequentially.
\end{remark}

	In the case where the state is multidimensional, we model each state transition using a separate GP, i.e.,
	\begin{align}
	\label{eq:OneStepPredMultiVar}
	\begin{split}
	&\x^{s}_{\step+1} \sim \mathcal{N} \left(\f(\xaug^{s}_{\step}) +\postmean^{s}_{\step}(\xaug^{s}_{\step}) ,(\bm{\Sigma}^{s}_{\step}(\xaug^{s}_{\step}))^2 + \bm{\Sigma}^2_w\right),
	\end{split}
	\end{align}
	where
	\begin{align*}
	\postmean^{s}_{\step}(\xaug^{s}_{\step}):=& \left(\postmeansc(\xaug^{s}_{\step}\vert \y_{1,\D^{s}_{\step}}) ,\  \dotsc ,\  \postmeansc(\xaug^{s}_{\step}\vert \y_{\dimx,\D^{s}_{\step}}) \right), \\
	\begin{split}
	(\bm{\Sigma}^{s}_{\step}(\xaug^{s}_{\step}))^2
	:= &\text{diag}\Big({\sigma}^2(\xaug^{s}_{\step}\vert \y_{1,\D^{s}_{\step}}), \ \dotsc, \ {\sigma}^2(\xaug^{s}_{\step}\vert \y_{\dimx, \D^{s}_{\step}})\Big).
	\end{split}
	\end{align*}
	Here~${{\y_{i,\D^{s}_{\step}} := \left( \gsc^{s}_i(\xaug^{s}_0) , \ \dotsc, \ \gsc^{s}_i(\xaug^{s}_{\step-1} ) \right)}}$
	concatenates samples of the~$i$-th component of the GP model for every~${i = 1, \dotsc, \dimx}$.\begin{remark}
		 Modeling each state transition with a separate GP corresponds to assuming that the state transitions are conditionally independent. Alternatively, a generalization of GPs to multiple dimensions is also applicable~\cite{Rasmussen2006}. However, the latter approach is significantly more computationally expensive than the former. Moreover, employing separate GPs for each state transition function has been shown to yield good results in practice~\cite{deisenroth2015gaussian}. 
	\end{remark}
	\begin{remark}
			For the sake of brevity, we only show here how to model a multidimensional~$\g(\cdot)$ using a single kernel~$\kernel(\cdot,\cdot)$ for all entries of~$\g(\cdot)$. However, the methods described herein extend straightforwardly to the case where different kernels are employed for each entry of~$\g(\cdot)$. 
	\end{remark}
	
	We assume that the model uncertainty due to the unknown function~$\g(\cdot)$ is faithfully captured by a GP with kernel~$\kernel(\cdot,\cdot)$. Formally, this is stated as follows.
	\begin{assumption}
		\label{assumption:gisSamplefromGP}
		Let~$\mathcal{GP}(\mean,\kernel)$ be a GP with mean~${\mean \equiv 0}$ and known continuously differentiable kernel~$\kernel(\cdot,\cdot)$. Then the entries of the unknown function~$\g(\cdot)$ are samples of~$\mathcal{GP}(\mean,\kernel)$, i.e.,~$\gsc_i \sim \mathcal{GP}(\mean,\kernel)$ holds for~$i=1, \dotsc, \dimx$.
	\end{assumption}
	
	Choosing an appropriate kernel~$\kernel(\cdot,\cdot)$ requires a priori knowledge of the system. However, the assumptions required for choosing a kernel are generally far less restrictive than for parametric models, since they only pertain to features such as smoothness and periodicity. Furthermore, in some cases, error bounds can be obtained if the kernel is poorly chosen~\cite{beckers2018misspecified}.
	
	
	\subsection{Predicting control performance}
	\Cref{assumption:gisSamplefromGP} implies that, for a fixed set of parameters~$\bm{\vartheta}$, the expected state of the true system~\eqref{eq:SystemDynamics} at an arbitrary time step~$\step$ is given by
	\begin{align}
	\label{eq:RewrittenEstDynamics}
	\begin{split}
	\text{E}_{\xaug_1, \dotsc, \xaug_{\NMPC}}\left[\x_{\step}\right] =& \int \limits_{\StatSp^t} \x^{s}_\step \prod \limits_{i=0}^{\step-1}  \text{p}(\bm{\zeta}^{s}_i)  d\bm{\zeta}^{s}_{i},
	\end{split}
	\end{align}
	where the integrand is computed recursively using
	\begin{align}
	\label{eq:MultivarNextState}
	\begin{split}
	\x^{s}_{i+1} = &\f(\xaug^{s}_{i}) + \bm{\mu}^{s}_{i}\left(\xaug^{s}_{i}\right)+ \left[\bm{\Sigma}^{s}_{i}\left(\xaug^{s}_{i} \right) \quad \bm{\Sigma}_w \right]\bm{\zeta}^{s}_{i},
	\end{split}
	\end{align}
	and~$\text{p}(\bm{\zeta}^{s}_{i}) = \mathcal{N}(\bm{0},\Id_{2\dimx})$.
	
	The corresponding cost function is given by
	\begin{align}
	\label{eq:IntegralExpCost}
	\C(\bm{\vartheta}) = \sum \limits_{\step=0}^{\NMPC} \int \limits_{\StatSp^t}  c_\step(\x^{s}_\step,\uin^{s}_\step(\bm{\vartheta})) \prod \limits_{i=0}^{\step-1} p(\bm{\zeta}^{s}_i)  d\bm{\zeta}^{s}_{i} .
	\end{align}
	\begin{lemma}
		\label{lemma:ContExpCost}
		Let \Cref{assumption:uiscontinuous,assumption:gisSamplefromGP} hold. Furthermore, let~$\Cintegrand(\x^{s}_\step,\uin^{s}_\step(\bm{\vartheta}))$ be the integrand of~\eqref{eq:IntegralExpCost}, where the states are computed using~\eqref{eq:MultivarNextState} and~${\bm{\zeta}^{s}_\step\sim \mathcal{N}(\bm{0}, \Id_{2\dimx})}$ for all~$\step$. Then both~$\Cintegrand(\x^{s}_\step,\uin^{s}_\step(\bm{\vartheta}))$ and~\eqref{eq:IntegralExpCost} are continuously differentiable with respect to $\bm{\vartheta}$.
	\end{lemma}
	\begin{proof}
		Since~$k(\cdot,\cdot)$, ~$c_{\step}(\cdot, \cdot)$,~$\uin^{s}_\step(\cdot)$ are continuously differentiable with respect to their arguments,~$\Cintegrand(\x^{s}_\step,\uin^{s}_\step(\bm{\vartheta}))$ is a composition of continuously differentiable functions. Hence it is continuously differentiable with respect to the control parameters~$\bm{\vartheta}$. Due to Leibniz's rule, this implies that~\eqref{eq:IntegralExpCost} is also continuously differentiable with respect to~$\bm{\vartheta}$.
	\end{proof}
	\section{Sample average approximation}
	\label{section:SAA}
	Computing the integral~\eqref{eq:IntegralExpCost} is generally intractable. Hence, we compute an estimate of the minimizer of~\eqref{eq:IntegralExpCost} by employing a \textit{sample average approximation} (SAA) of~\eqref{eq:IntegralExpCost},
	\begin{equation}
	\label{eq:SAAExpCost}
	\medmuskip=2mu
	\thinmuskip=2mu
	\thickmuskip=2mu
	\C(\bm{\vartheta}) \simeq  \C_{\nMC}(\bm{\vartheta}, \mathcal{Z}_M) :=  \frac{1}{\nMC}\sum \limits_{{\step}=0}^{\NMPC} \left( \sum  \limits_{m=1}^\nMC c_\step\left(\x^{(m)}_{\step},\uin^{(m)}_{\step}(\bm{\vartheta})\right) \right).
	\end{equation}
	Here~${\nMC \in \mathbb{N}}$ is the number of sample trajectories. The set~${\mathcal{Z}_M:= \left\{\bm{\zeta}_{0}^{(1)}, \dotsc,\bm{\zeta}_{\NMPC-1}^{(1)},\dotsc,\bm{\zeta}_{0}^{(M)}, \dotsc,\bm{\zeta}_{\NMPC-1}^{(M)}\right\}}$ subsumes~$M \NMPC$ samples from~$\mathcal{N}(\bm{0},\Id_{2\dimx})$, which are treated as fixed quantities during optimization. The superscript~$(m)$ denotes the~$m$-th sample trajectory, which is computed recursively as
	\begin{align*}
	\begin{split}
	\x_{\step+1}^{(m)} =& \f\left(\xaug_{\step}^{(m)}\right) + \bm{\mu}^{s}_{i}\left(\xaug^{(m)}_{\step} \right)  + \left[\bm{\Sigma}^{s}_{i}\left(\xaug^{(m)}_{\step} \right) \ \bm{\Sigma}_w\right]  \bm{\zeta}^{(m)}_{\step},
	\end{split}
	\end{align*} 
	with~${\xaug^{(m)}_{\step}\coloneqq (\x^{(m)}_{\step},\uin^{(m)}_{\step}(\bm{\vartheta}))}$,~${\D_{\step}^{(m)} := \{\xaug_0^{(m)}, \dotsc, \xaug_{\step-1}^{(m)}\}}$. 
	We denote the minimum of the SAA~\eqref{eq:SAAExpCost} as~${\C_{\nMC}^* := \min_{\bm{\vartheta} \in \bm{\varTheta}}\C_{\nMC}(\bm{\vartheta}, \mathcal{Z}_M)}$, and the corresponding set of minimizers as~${\ParamSp_M^* := \left\{ \bm{\vartheta} \in \bm{\varTheta} \ \vert \  \C_{\nMC}(\bm{\vartheta}, \mathcal{Z}_M) = \C_{\nMC}^*\right\}}$.
	
	\begin{algorithm}[b]
		\caption{Anticipating learning (AntLer)}
		\label{alg:AntLer}
		\begin{algorithmic}
			\Require{$\x_0$, $\uin(\cdot, \cdot,\cdot)$, $\NMPC$, $M$, $\bm{\Sigma}_w$, $\f(\cdot)$, $\bm{\zeta}^{(1)}_0, \dotsc, \bm{\zeta}^{(\nMC)}_{\NMPC-1}$}
			\State{Solve 
				\begin{align*}
				\bm{\vartheta}^*_M =&  \arg \min\limits_{\bm{\vartheta}}  \sum \limits_{{\step}=0}^{\NMPC} \left(\frac{1}{\nMC} \sum  \limits_{m=1}^\nMC c_\step\left(\x^{(m)}_{\step},\uin^{(m)}_{\step}(\bm{\vartheta})\right) \right) \\
				\begin{split} \text{s.t.} \quad & 
				\x_{\step+1}^{(m)} = \f\left(\xaug_{\step}^{(m)}\right) + \bm{\mu}^{s}_{i}\left(\xaug^{(m)}_{\step} \right)  \\ & \qquad \quad + \left[\bm{\Sigma}^{s}_{i}\left(\xaug^{(m)}_{\step} \right) \quad \bm{\Sigma}_w\right]  \bm{\zeta}^{(m)}_{\step} \quad \\
				&\xaug^{(m)}_{\step}= (\x^{(m)}_{\step},\uin^{(m)}_{\step}(\bm{\vartheta})) \\
				& \xaug_0^{(m)} = \left(\x_0, \uin_0(\bm{\vartheta})\right) \\&
				\D_{\step}^{(m)} = \{\xaug_0^{(m)}, \dotsc, \xaug_{\step-1}^{(m)}\} \\
				& \forall \ \step\in \left\{0,\dotsc,\NMPC-1\right\}, \ m\in \left\{1, \ \dotsc,  \nMC\right\} 
				\end{split}
				\end{align*}}
			\Ensure{$\bm{\vartheta}^*_M$}
		\end{algorithmic}
	\end{algorithm}
	The steps required to compute a minimizer of \eqref{eq:SAAExpCost} yield the AntLer algorithm, which is presented in \Cref{alg:AntLer}. 
	
	\begin{remark}
		Despite being mainly designed with online learning-based control laws in mind, AntLer can also be employed in the special case where the control law does not change based on the data collected online. In such settings, AntLer becomes similar in principle to model-based reinforcement learning approaches, e.g.,~\cite{deisenroth2015gaussian}.
	\end{remark}
	
			\begin{remark}
		The algorithm proposed in this paper can also be applied to the infinite-horizon case, e.g., by implementing it in a receding horizon fashion. This would generally require a terminal constraint to be considered, for which probabilistic guarantees can be derived, e.g., as in~\cite{armin2020confidence}.
	\end{remark}
	
	We now aim to prove that a solution~$\bm{\vartheta}_M^*$ obtained with AntLer approximates an optimum~$\bm{\vartheta}^* \in \ParamSp^*$ of the exact problem arbitrarily accurately for a sufficiently high number of samples~$\nMC$. To achieve this, we show that both the approximate and exact cost functions~$\C_{\nMC}(\cdot,\cdot)$,~$\C(\cdot)$, satisfy some regularity conditions.
	
	\begin{lemma}
		\label{lemma:CiscontdiffandCMconverges}
		Let \Cref{assumption:GrowthofCosts,assumption:gisSamplefromGP,assumption:uiscontinuous} hold, and choose~$\tilde{\ParamSp}$ as in \Cref{assumption:GrowthofCosts}.  Then~$\C(\cdot)$ is finite-valued and continuously differentiable on~$\tilde{\ParamSp}$, and~$\C_{\nMC}(\cdot,\mathcal{Z}_M)$ converges to~$\C(\cdot)$ with probability~$1$ uniformly in~$\tilde{\ParamSp}$ as~${\nMC \rightarrow \infty}$. 
	\end{lemma}

	To prove \Cref{lemma:CiscontdiffandCMconverges}, we make use of the following result, which corresponds to~\cite[Theorem 7.48]{shapiro2009lectures}:
	\begin{lemma}[\cite{shapiro2009lectures}]
		\label{theorem:UniformConvergence}
		Let~$\tilde{\ParamSp}$ be a nonempty compact subset of~$\ParamSp$ and suppose that
		\begin{enumerate}[i)]
			\item For any~$\bm{\vartheta} \in \tilde{\ParamSp}$, the function~$\Cintegrand(\x^{s}_\step,\uin^{s}_\step(\bm{\vartheta}))$ is continuously differentiable at~$\bm{\vartheta}$ for almost every sample~$\bm{\zeta}^{(m)}_\NMPC$,
			\item The absolute value of~${\Cintegrand(\x^{s}_\step,\uin^{s}_\step(\bm{\vartheta}))}$ is upper bounded by an integrable function on the subset~$\tilde{\ParamSp}$,
			\item The samples~${\bm{\zeta}^{(m)}_0, \dotsc, \bm{\zeta}^{(m)}_{\NMPC-1}}$ are i.i.d.
		\end{enumerate}
		Then~$\C(\cdot)$ is finite-valued and continuously differentiable on~$\tilde{\ParamSp}$, and~$\C_{\nMC}(\cdot,\mathcal{Z}_M)$ converges to~$\C(\cdot)$ with probability~$1$ uniformly in~$\tilde{\ParamSp}$ as~$\nMC \rightarrow \infty$. 
	\end{lemma}
\begin{proof}[Proof of \Cref{lemma:CiscontdiffandCMconverges}]
	We show that the conditions of \Cref{theorem:UniformConvergence} hold for the compact subset~$\tilde{\ParamSp}$ from \Cref{assumption:GrowthofCosts}.
	
	Since~$\tilde{\ParamSp}$ is bounded,  \Cref{lemma:ContExpCost} implies that~$\Cintegrand(\x^{s}_\step,\uin^{s}_\step(\bm{\vartheta}))$ satisfies conditions i) and ii) of \Cref{theorem:UniformConvergence}. Moreover, the samples~$\bm{\zeta}^{(m)}_1, \dotsc, \bm{\zeta}^{(m)}_\NMPC$ are i.i.d., i.e., condition iii) of \Cref{theorem:UniformConvergence} is also satisfied.
\end{proof}

Using \Cref{lemma:CiscontdiffandCMconverges}, we are able prove that \Cref{alg:AntLer} approximates an optimal solution~${\bm{\vartheta}^*\in \ParamSp^*}$ arbitrarily accurately with probability~$1$ for large enough $\nMC$. 
This corresponds to our main result, and is stated in the following theorem.

\begin{theorem}
	\label{theorem:MainResult}
	Let~${\bm{\zeta}_{\step}^{(m)} \sim \mathcal{N}(\bm{0},\bm{I}_{\dimx})}$,~${\step\in \left\{0,\dotsc,\NMPC-1\right\}}$, ${m\in \left\{1,\dotsc,\infty\right\}}$ be a fixed sequence of random samples. For every $\nMC$, let~$\bm{\vartheta}^*_{\nMC}$ denote a vector of approximate optimal parameters obtained with \Cref{alg:AntLer} and the samples ${\bm{\zeta}_{0}^{(1)},\dotsc, \bm{\zeta}_{\NMPC-1}^{(\nMC)}}$. Moreover, let \Cref{assumption:GrowthofCosts,assumption:uiscontinuous,assumption:gisSamplefromGP} hold. Then, for every~${\epsilon>0}$, there exists an~$\nMC_{\epsilon} \in \mathbb{N}$, such that~${\lvert \C_{\nMC}^* -  \C^* \rvert \leq \epsilon}$ and~${ \min_{\bm{\vartheta}^* \in \bm{\varTheta}^*} \lVert \bm{\vartheta}_M^* - \bm{\vartheta}^*\rVert_2 \leq \epsilon}$ holds for all~$\nMC \geq \nMC_{\epsilon}$ with probability~$1$. 
\end{theorem}

We prove \Cref{theorem:MainResult} by employing~\cite[Theorem 5.3]{shapiro2009lectures}, which we now state.

	\begin{lemma}[\cite{shapiro2009lectures}]
		\label{theorem:Shapiro}
		Suppose there exists a compact subset~$\tilde{\bm{\varTheta}} \subseteq \bm{\varTheta}$, such that
		\begin{enumerate}[i)]
			\item~$\bm{\varTheta}^*$ is non-empty and~$\bm{\varTheta}^* \subseteq \tilde{\bm{\varTheta}}$,
			\item The function~$\C(\bm{\vartheta})$ is finite-valued and continuously differentiable on~$\tilde{\bm{\varTheta}}$,
			\item~$\C_{\nMC}(\bm{\vartheta},\mathcal{Z}_M)$ converges to~$\C(\bm{\vartheta})$ with probability~$1$ as~$M \rightarrow \infty$, uniformly in~$\bm{\vartheta} \in \tilde{\bm{\varTheta}}$,
			\item With probability~$1$, for~$M$ large enough, the set~$\bm{\varTheta}^*_M$ is nonempty and~$\bm{\varTheta}^*_M \subseteq \tilde{\bm{\varTheta}}.$
		\end{enumerate} 
		Then~${\C_{\nMC}^*  \rightarrow \C^*}\medmuskip=0mu
		\thinmuskip=0mu
		\thickmuskip=1mu$, ${\max_{\bm{\vartheta_M}^* \in \bm{\varTheta}_M^*} \min_{\bm{\vartheta}^* \in \bm{\varTheta}^*} \lVert \bm{\vartheta}_M^* - \bm{\vartheta}^*\rVert_2 \rightarrow 0}$ holds with probability~$1$ as~$M \rightarrow \infty$.
	\end{lemma}

	\begin{proof}[Proof of \Cref{theorem:MainResult}]
		We show that the conditions of \Cref{theorem:Shapiro} hold for the compact subset~$\tilde{\ParamSp}$ from \Cref{assumption:GrowthofCosts}.

		Conditions ii) and iii) are satisfied due to \Cref{lemma:CiscontdiffandCMconverges}. Hence, it remains to be shown that i) and iv) hold. 
		
		We begin by showing that the set~$\ParamSp^*$ is nonempty. To this end, consider an arbitrary sequence of control parameters~$\bm{\vartheta}_i$,~${i=1, \dotsc, \infty}$, with~${\lim_{i \rightarrow \infty} \C(\bm{\vartheta}_i) = \C^*}$. Due to \Cref{assumption:GrowthofCosts} and the continuity of~$\C(\cdot)$ (i.e., \Cref{lemma:ContExpCost}), there exists an~$I \in \mathbb{N}$, such that~$ \bm{\vartheta}_i \in \tilde{\ParamSp}$ holds for all~${i \geq I}$. Since~$\bm{\vartheta}_{1}, \dotsc, \bm{\vartheta}_{I}$ are finite-valued, this implies that the sequence~$\bm{\vartheta}_{i}$,~${i=1, \dotsc, \infty}$ belongs to a compact set. Due to the Bolzano-Weierstrass theorem,~$\bm{\vartheta}_{i}$ contains a convergent subsequence with limit~$\bm{\vartheta}^* \in \tilde{\ParamSp}$. Hence,~$\bm{\varTheta}^*$ is nonempty and~$\bm{\varTheta}^* \subseteq \tilde{\bm{\varTheta}}$, i.e., Condition i) of \Cref{theorem:Shapiro} is satisfied. 
		Using the same argument we can show that~$\bm{\varTheta}^*_{\nMC}$ is nonempty. Moreover, \Cref{assumption:GrowthofCosts} implies that~$\bm{\varTheta}^*_M\subseteq \tilde{\bm{\varTheta}}$ holds for all~$\mathcal{Z}_M$, i.e., Condition iv) is satisfied.
	\end{proof}
	
	Hence, 
	the AntLer algorithm  approximates an optimal vector of data-independent parameters~$\bm{\vartheta}^*$ with arbitrary accuracy for large enough~$\nMC$ with probability~$1$.
	 
	For a control law $\uin_{\step}(\bm{\vartheta})$ that potentially improves its performance through online learning, \Cref{theorem:MainResult} implies that AntLer guarantees superior control performance for large enough~$\nMC$ compared to the case where learning is not anticipated. This is shown by comparing predictions for~$\uin_{\step}(\cdot)$ to predictions for its data independent counterpart~${\uin^0_{\step}(\cdot)=(\D_0,\cdot,\bm{\vartheta})}$. We state this formally in the following.
	
	\begin{assumption}
		\label{assumption:DataIndependentCounterpartisWorse}
		Let
		\begin{align}
		\label{eq:CostFunDataIndependent}
		\C^0(\bm{\vartheta}) \coloneqq  \text{E}_{\xaug_1, \dotsc, \xaug_{\NMPC}}\left[{\sum \limits_{{\step}=0}^{\NMPC} c_\step(\x_{\step},\uin(\D_0,\x_{\step},\bm{\vartheta}))}\right],
		\end{align} 
		be the cost function under the data-independent counterpart~$\uin^0_{\step}(\bm{\vartheta})=\uin(\D_0, \bm{\vartheta}, \x_{\step})$, and let~${{\C^{0,*}\coloneqq \min_{\bm{\vartheta}}\C^0(\bm{\vartheta})}}$ be its minimum. Then~$\C^{0,*} < \C^*$, where~$\C^*$ is the minimum of \eqref{eq:CostFun}.
	\end{assumption}
	
	This amounts to assuming that~$\uin_{\step}(\bm{\vartheta})$ potentially improves its performance as new data is gathered.
	
	\begin{corollary}
		\label{corollary:AmBetterthanwithoutAnticipating}
		Let \Cref{assumption:GrowthofCosts,assumption:gisSamplefromGP,assumption:uiscontinuous,assumption:DataIndependentCounterpartisWorse} hold,
		and let~${\C(\cdot)}$ be given as in \eqref{eq:CostFun}. Furthermore, let~${\C^{0,*}}$ be the optimal cost under the data-independent counterpart, as given in \Cref{assumption:DataIndependentCounterpartisWorse}, and let~${\bm{\zeta}_{\step}^{(m)} \sim \mathcal{N}(\bm{0},\bm{I}_{\dimx})}$,~${\step\in \left\{0,\dotsc,\NMPC-1\right\}}$,~${m\in \left\{1,\dotsc,\infty\right\}}$, be a fixed sequence of random samples. For every~$\nMC \in \mathbb{N}$, let~$\bm{\vartheta}^{*}_{\nMC}$ denote the approximate optimal solution obtained with \Cref{alg:AntLer} and the samples~${\bm{\zeta}_{0}^{(1)}, \dotsc, \bm{\zeta}_{\NMPC-1}^{(\nMC)}}$. Then there exists an~$\nMC^0$, such that~${\C(\bm{\vartheta}^{*}_{\nMC}) < \C^{0,*}}$ holds for all~${\nMC \geq \nMC^0}$ with probability~$1$.
	\end{corollary}
	\begin{proof}
		This follows directly from \Cref{theorem:MainResult}.
	\end{proof}
	\section{Numerical example}
	\label{section:Example}
	
	We now illustrate the proposed algorithm using a simple nonlinear trajectory tracking problem. We demonstrate the convergence of the approximate optimal parameters computed by AntLer as the number of samples $\nMC$ grows, and compare the computed parameters and predictions to those obtained without anticipating learning. Furthermore, by preforming Monte Carlo simulations of the true system, we showcase the superior performance of the approximate optimal parameters compared to the case where learning is not anticipated.
	
	The source code of the experiments presented in this section is available at https://git.lsr.ei.tum.de/acapone/antler.
	
	\subsection{System description}
	
	We consider the one-dimensional system
	\begin{align}
	\label{eq:ToyProbDiffEq}
	\xsc_{\step+1} = \fsc(\xaugsc_{\step} ) +  \gsc(\xaugsc_{\step}) + \wsc_{\step},
	\end{align}
	with initial state~$\xsc_0 = 0$, process noise~$\wsc_\step \sim \mathcal{N}(0,\covtoyproblem^2)$, and state transition functions
	 \begin{align}
	\fsc(\xaugsc_{\step}) &= \xsc_{\step} + \uinsc_{\step}, \\
	\gsc(\xaugsc_{\step}) &= \gsctoyprxsc.
	\end{align} 
	We aim to design an online learning-based control law that tracks the trajectory~${\xsc_{\step}^{\text{ref}} = \xdestoyprob}$ as accurately as possible, while simultaneously accounting for any potential tracking errors due to the unknown function $\gsc(\cdot)$. To this end we choose the control law
	\begin{align}
	\label{eq:ControlLawToyProblem}
	{\uinsc_{\step}(\bm{\vartheta}) = -  \mu_\step(\xsc_{\step}) - \vartheta_1( \xsc_{\step} - \vartheta_2 \xsc_{\step}^{\text{ref}}  )},
	\end{align}
	where~$\vartheta_1$ acts as a control gain, and~$\vartheta_2$ scales the reference trajectory and enables to avoid regions of high model uncertainty. The term~$\mu_\step(\xsc_{\step})$ is a GP mean, which is updated online as new data points are collected. We compute~$\mu_\step(\xsc_{\step})$ using the same kernel as for AntLer, which we specify in the sequel. Employing the same kernel both for predictions and control is reasonable, since we assume that it faithfully represents the unknown function~$g(\cdot)$.
	
	We quantify control performance by employing the cost function 
	\begin{align}
	\label{eq:CostFunToyProb}
	\begin{split}
	\C({\bm{\vartheta}}) = &\text{E}_{\xaug_1, \dotsc, \xaug_{\nHtoyprob}} \left[\sum \limits_{{\step}=0}^\nHtoyprob
	c_{\step} \right],
	\end{split}
	\end{align}  
	where the immediate cost terms~$c_{\step}	\coloneqq (\xsc_{\step}-\xsc_{\step}^{\text{ref}})^2$ penalize deviations from the reference trajectory. 
		
	We now describe the kernel used for AntLer predictions and the online learning-based control law \eqref{eq:ControlLawToyProblem}. We assume to know that~$\gsc(\cdot)$ depends only on the state~$\xsc_{\step}$, and that it corresponds to a smooth function. This information is encoded into the GP by employing a squared exponential kernel that takes only the state as argument, i.e.,
	\begin{align}
	\label{eq:SqExpKernel}
	k(\xaugsc_i, \xaugsc_j) \eqqcolon k(\xsc_i, \xsc_j) = \sigma^2_k \exp \left(- \frac{(\xsc_i- \xsc_j)^2}{2 l} \right),
	\end{align}
	where the signal variance~${\sigma_k^2 \in \mathbb{R}_+}$ and length scale~${l \in \mathbb{R}_+}$ are obtained by training the GP using log marginal likelihood optimization~\cite{Rasmussen2006}. To this end, we assume to have~$100$ measurements of~\eqref{eq:ToyProbDiffEq}, which were obtained using a control law that attempts to minimize the distance of the true system~\eqref{eq:ToyProbDiffEq} to the origin.
	Squared exponential kernels are dense within the space of continuous functions on compact sets, i.e., they can approximate any continuous function uniformly and arbitrarily well on compact subsets of~$\StatSp$~\cite{micchelli2006universal}. Moreover, the posterior mean~$\mu_\step(\cdot)$ of a GP obtained with a squared exponential kernel exhibits smooth behavior~\cite{Rasmussen2006}. Hence,~\eqref{eq:SqExpKernel} is an appropriate choice for this setting. 
	
	It can easily be shown that, in a setting where~$\gsc(\xsc_{\step})$ is known, i.e.,~$\mu_\step(\xsc_{\step}) =  \gsc(\xsc_{\step})$, the system trajectory is optimal for~$\vartheta_1 =\vartheta_2=1$. Since the control law ~\eqref{eq:ControlLawToyProblem} learns~$\gsc(\cdot)$ online, it is reasonable to expect that the optimal parameters~$\bm{\vartheta}^*$ for unknown~$\gsc(\cdot)$ lie within a neighborhood of~$\vartheta_1 =\vartheta_2=1$, provided that~$\gsc(\cdot)$ is learned correctly. Hence, we assume that the optimal parameters lie within the compact subset $\tilde{\ParamSp}= [-1,-1] \times[2,2]$. In the following, we employ this assumption to restrict the feasible region of the optimization problem to $\tilde{\ParamSp}$.

	\subsection{Approximate optimal solutions using AntLer}
	\label{subsection:OptimControlUsingSAA}
	
		\begin{figure*}[ht]
			\vspace{0.14cm}
		\centering
		\begin{subfigure}[t]{0.486\textwidth}
			\centering
			\includegraphics[scale=0.285]{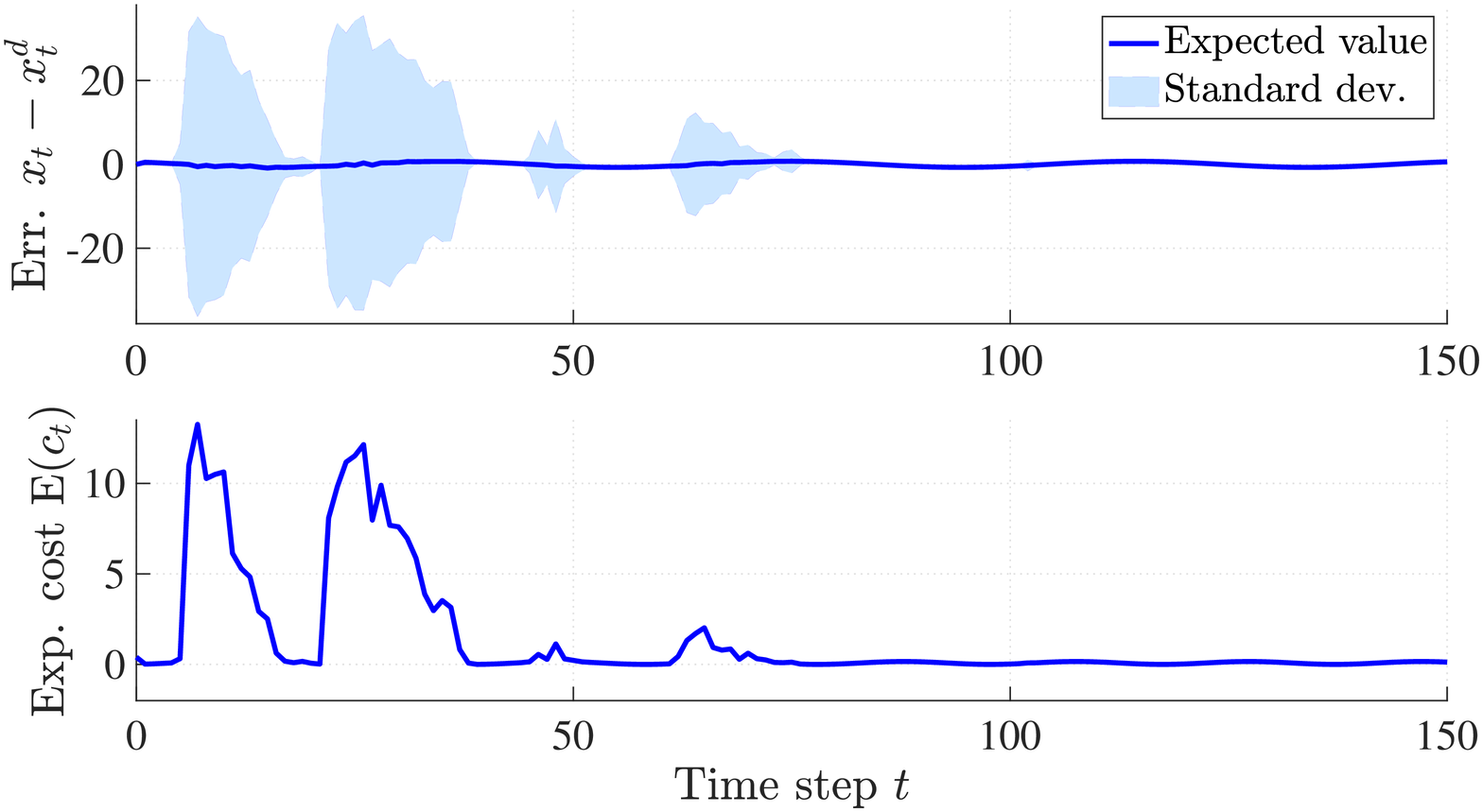}
			\caption{Predicted optimum when learning is anticipated. } 
			\label{fig:ToyProblemPredictionWithLearning}
		\end{subfigure}%
		\hfill
		\begin{subfigure}[t]{0.486\textwidth}
			\centering
			\includegraphics[scale=0.285]{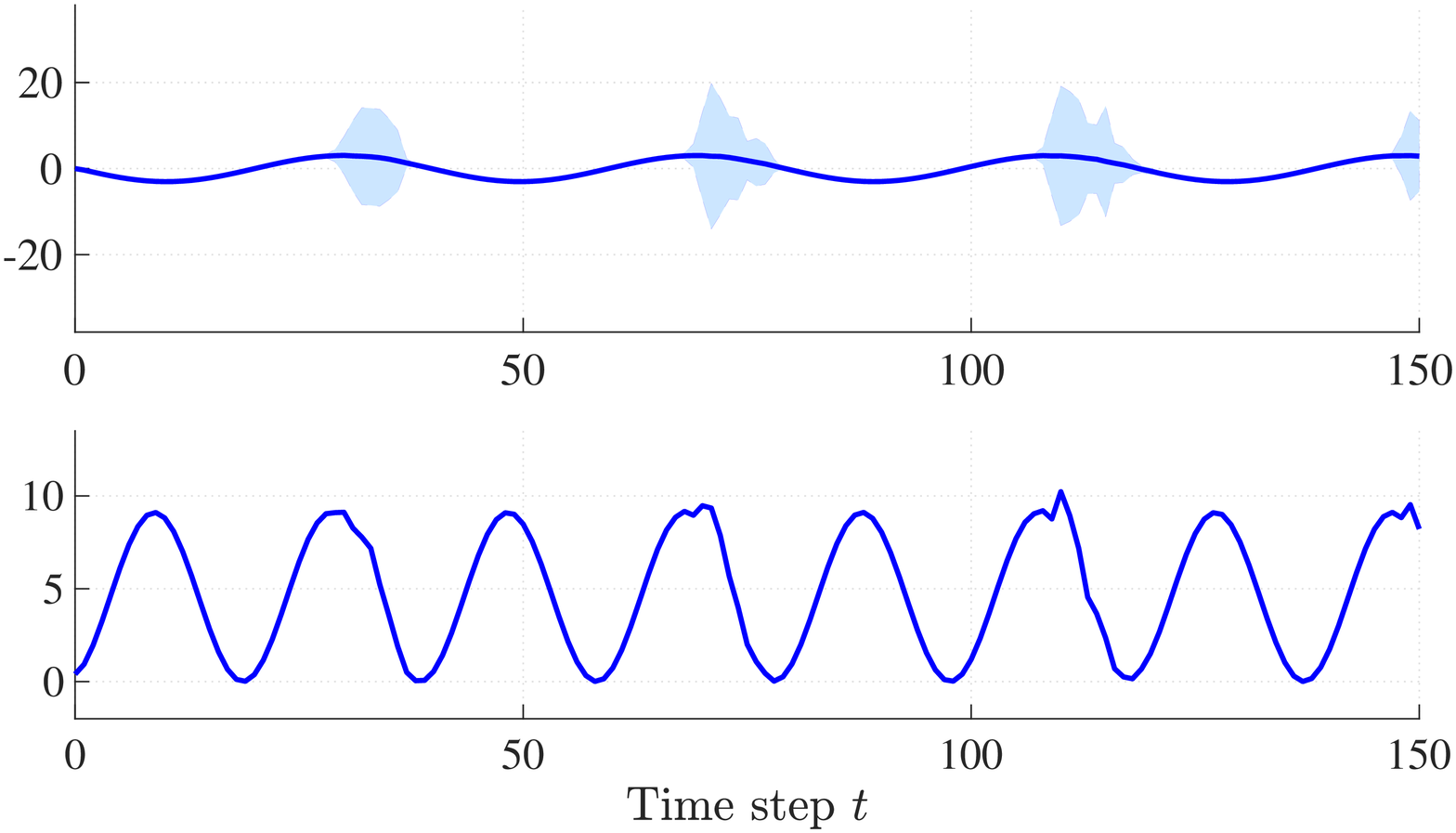}
			\caption{Predicted optimum when learning is not anticipated. }
			\label{fig:ToyProblemPredictionWithoutLearning}
		\end{subfigure}%
		\caption{Predicted optima (\subref{fig:ToyProblemPredictionWithLearning})~with and (\subref{fig:ToyProblemPredictionWithoutLearning})~without anticipating learning. Both predictions are carried out using AntLer and~${\nMC = 200}$ samples. The top rows show the predicted optimal tracking error~$x_t-x_t^{\text{ref}}$, the bottow rows show the expected immediate cost~$\text{E}[c_t]$. (\subref{fig:ToyProblemPredictionWithLearning})~Prediction for approximate optimal online learning-based control law~${\uinsc_{\step}(\bm{\vartheta}_{\nMC}^* )}$, where ${{\bm{\vartheta}_{\nMC}^*=(1, 0.9)\transp}}$; predicted cost is~${{\C}_M(\bm{\vartheta}_{\nMC}^*, \mathcal{Z}_{\nMC})=212}$. (\subref{fig:ToyPWithoutLearning})~Prediction for approximate optimal data-independent counterpart~${{\uinsc^0_{\step}(\bm{\vartheta}^{0,*} )}}$, where~${{\bm{\vartheta}^{0,*}=(0.9, 0.2)\transp}}$; predicted cost is~${{\C}_M(\bm{\vartheta}_{\nMC}^*, \mathcal{Z}_{\nMC}) =731}$.}
		\label{figure:ToyProblemAntlerPredictions}
	\end{figure*} 
	
		We demonstrate the convergence of the approximate optimal solution~$\bm{\vartheta}_{\nMC}^*$ to~$\bm{\vartheta}^*$ as~$\nMC$ grows by applying AntLer using~${\nMC\in\left\{2,10,50,100,200\right\}}$ samples. Additionally, in order to illustrate
	\Cref{corollary:AmBetterthanwithoutAnticipating}, we make predictions and optimize the parameters~$\bm{\vartheta}$ without anticipating learning, i.e., by using the data-independent counterpart of the control law~${\uinsc^0_{\step}(\bm{\vartheta}) =  - \mu_0(\xsc_{\step}) -\vartheta_1(\xsc_{\step} -  \vartheta_2\xsc_{\step}^{\text{ref}}})$. To optimize the parameters of~${\uinsc^0_{\step}(\bm{\vartheta})}$, we employ AntLer with $\nMC=200$. We are able to do so, since~${\uinsc^0(\cdot)}$ is a special case of an online learning-based control law. Hence, approximate optimal parameters can also be obtained using AntLer. For simplicity of exposition, we henceforth refer to~${\bm{\vartheta}^{*,0}}$  as the optimum of the data-independent counterpart.
		
	To solve the SAA problem in AntLer, we employ a gradient-based method with different starting values, which are sampled from the uniform distribution on $\tilde{\ParamSp}$. 
	\begin{table}[b]
		\caption{Approximate optimal parameters~$\bm{\vartheta}^{*}_{\nMC}$ 
			computed by AntLer for different~$\nMC$.}
		\label{table:CasfunctionofM}
		\begin{center}
			\small\addtolength{\tabcolsep}{-3.7pt}
			\begin{tabular}{|c||c|c|c|c|c|}
				\hline\
				$\nMC$ & 2 & 10 & 50 & 100 & 200 \rule{0pt}{2.6ex} \\ [0.5ex]
				\hline \hline
				$\bm{\vartheta}_{\nMC}^*$  & $(
				0.9, 0.9
				)\transp$ & $  (1.1, \ 1)\transp$ & $(1, 0.9)\transp$ & $(1,  0.9)\transp$&$(1,  0.9)\transp$\rule{0pt}{2.8ex}\\ [0.5ex]
				\hline
			\end{tabular}
		\end{center}\vspace{-2mm}
	\end{table}
A solution is found after at most~$17$ gradient-descent steps. In \Cref{table:CasfunctionofM}, we display the approximate optimal parameters~$\bm{\vartheta}^*_M$ 
computed by Antler. In \Cref{fig:ToyProblemPredictionWithLearning}, we present AntLer predictions for~$\nMC=200$ and the approximate optimal online learning-based law~$\uinsc_{\step}(\bm{\vartheta}^{*}_M)$. Furthermore, in \Cref{fig:ToyProblemPredictionWithoutLearning} we show predictions for the optimal data-independent counterpart~$\uinsc^0_{\step}(\bm{\vartheta}^{0,*})$. 
	
	The value of the approximate optimal parameters is~${\bm{\vartheta}_{\nMC}^*\approx (1,0.9)\transp}$ for all ${50<\nMC<200}$. This indicates that~$\bm{\vartheta}_{\nMC}^*$ has converged to a small neighborhood of the optimal parameters~$\bm{\vartheta}^*$, as expected from \Cref{theorem:MainResult}. 
	
	AntLer predicts that, by scaling the reference trajectory with~$\vartheta_2=0.9$, an optimal trade-off is achieved between the information of the collected data and the error caused by model uncertainty. In other words, if the control law were to attempt to fully enforce the reference trajectory, i.e.,~${\vartheta_2 = 1}$, then AntLer predicts that too many measurements need to be collected before good tracking performance is achieved. However, if~${\vartheta_2 = 0.9}$ is chosen, then AntLer predicts that the unknown dynamics will be learned quickly enough to achieve good tracking performance within the time horizon~${\NMPC = 150}$. This becomes apparent in the predictions in \Cref{fig:ToyProblemPredictionWithLearning}. Therein, the variance of the state $\xsc_{\step}$ and the expected immediate cost~$c_{\step}$ under the approximate optimal control law~$\uinsc_{\step}(\bm{\vartheta}_{\nMC}^{*})$ decrease over time. After~$\step \approx 70$, they become approximately zero. 
	
	The parameters~${\bm{\vartheta}^{0,*}=(0.9, 0.2)\transp}$ of the optimal data-independent counterpart~$\uinsc_{\step}(\bm{\vartheta}^{0,*})$ attempt to keep the system close to the origin. 
	This is because predictions for~$\uinsc_{\step}^0(\bm{\vartheta})$ do not anticipate learning. In other words, they only yield low tracking errors in regions where model uncertainty is already low. As measurement data at the origin was collected prior to the control design, model uncertainty is high in the whole state space except for a neighborhood of the origin. Hence the approximate optimal parameters~${\bm{\vartheta}^{0,*}=(0.9, 0.2)\transp}$ attempt to keep the system within this region. This is reflected in the predictions in \Cref{fig:ToyProblemPredictionWithoutLearning}, where the tracking error exhibits little variance compared to \Cref{fig:ToyProblemPredictionWithLearning}. 
	
	For~$\nMC=200$, the predicted cost~${\C(\bm{\vartheta}_{\nMC}^{*})=212}$ under the approximate optimal control law~$\uinsc_{\step}(\bm{\vartheta}_{\nMC}^{*})$ is lower than the predicted cost~$\C(\bm{\vartheta}^{0,*})=731$ under the data-independent counterpart~$\uinsc_{\step}(\bm{\vartheta}^{0,*})$. Assuming that the GP specified by the kernel \eqref{eq:SqExpKernel} correctly captures the model uncertainty due to~$\gsc(\cdot)$, \Cref{corollary:AmBetterthanwithoutAnticipating} implies that control performance will be superior if~$\bm{\vartheta}_{\nMC}^*$ is applied to the true system instead of~$\bm{\vartheta}^{0,*}$. This indeed is the case, as shown in the following. 

	\subsection{Monte Carlo simulations of true system}
	\label{subsect:MCSimofTrueSystem}
	
	\begin{figure*}[t]
		\vspace{0.14cm}
		\centering
		\begin{subfigure}[t]{0.486\textwidth}
			\centering
			\includegraphics[scale=0.285]{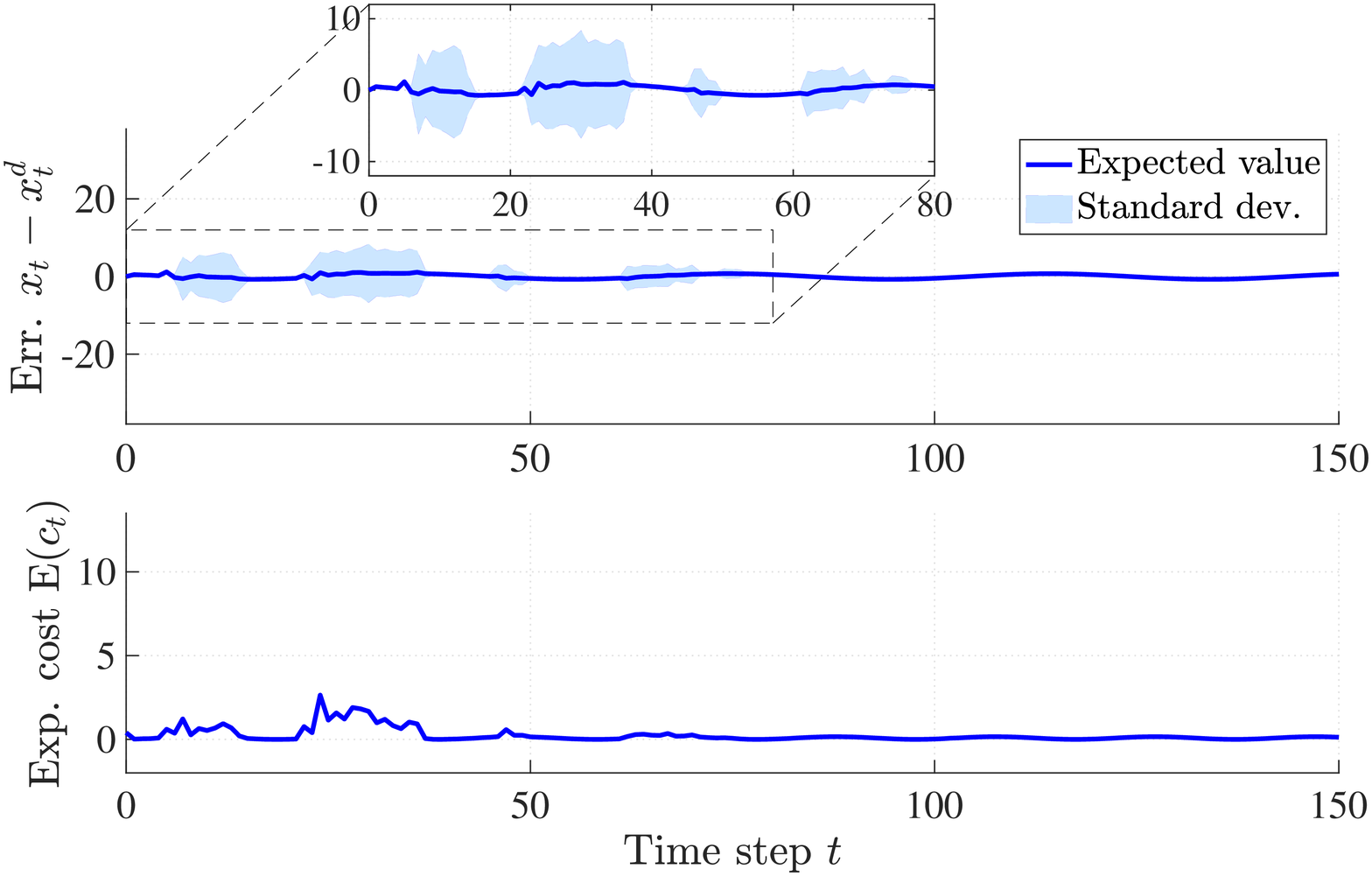}
			\caption{Parameters optimized by anticipating learning. } 
			\label{fig:ToyPWithLearning}
		\end{subfigure}%
		\hfill
		\begin{subfigure}[t]{0.486\textwidth}
			\centering
			\includegraphics[scale=0.285]{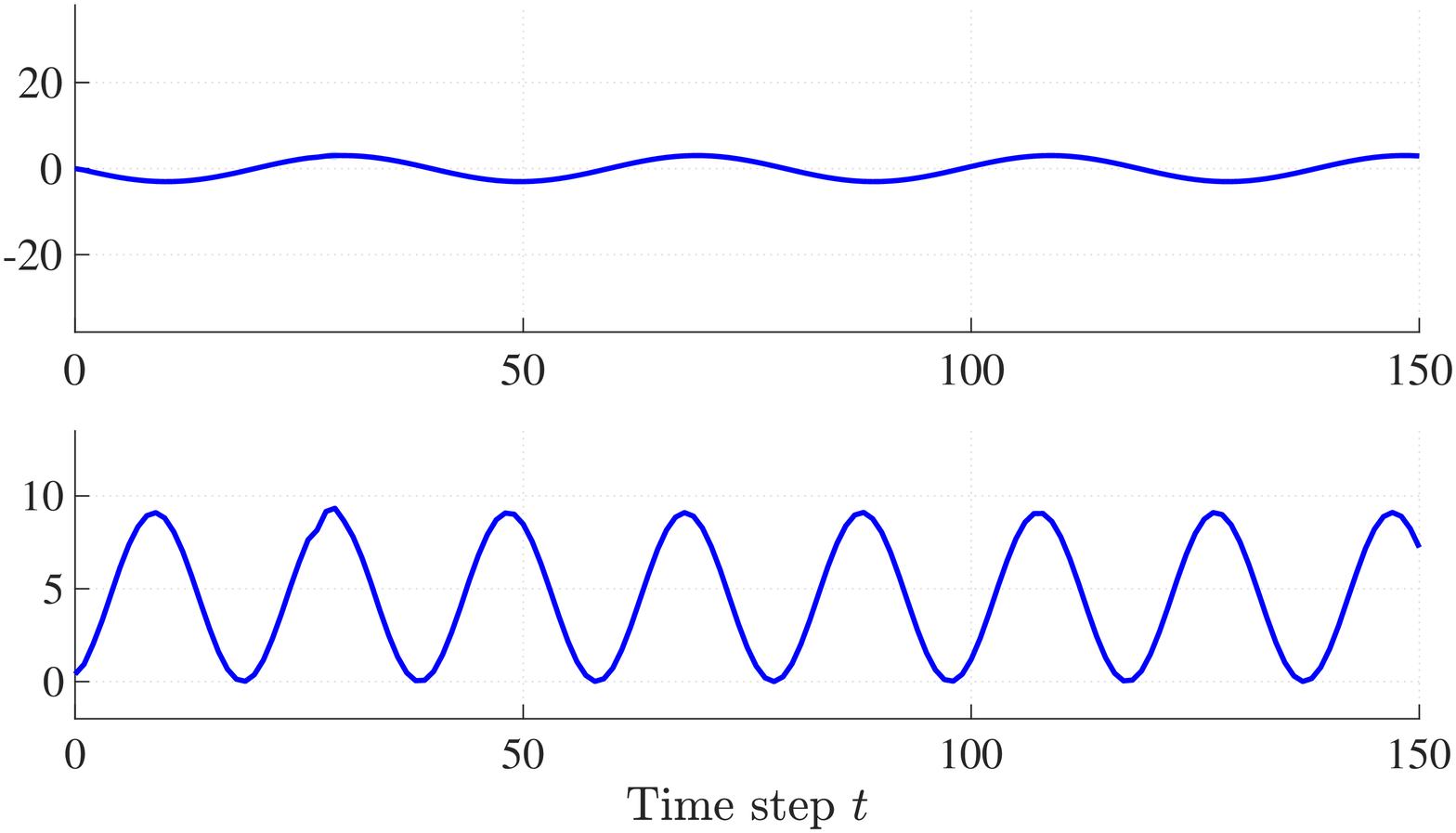}
			\caption{Parameters optimized without anticipating learning.} 
			\label{fig:ToyPWithoutLearning}
		\end{subfigure}
		\caption{Monte Carlo simulations of true system~\eqref{eq:ToyProbDiffEq} consisting of~$100$ runs. The top rows show the tracking error~$x_t-x_t^{\text{ref}}$, the bottom rows show the average immediate cost~$\text{E}[c_t]$. (\subref{fig:ToyProblemPredictionWithLearning})~True system under approximate optimal online learning-based control law~$\uinsc_{\step}(\bm{\vartheta}_{\nMC}^* )$, where~${\bm{\vartheta}_{\nMC}^* = (1, \ 0.9)\transp}$ was obtained by anticipating learning; the average total cost is~$37$. (\subref{fig:ToyPWithoutLearning})~True system under~$\uinsc_{\step}(\bm{\vartheta}^{0,*} )$, where ~${\bm{\vartheta}^{0,*} = (0.9, \ 0.2)\transp}$ was obtained without anticipating learning; the average total cost is~$859$.}
	\end{figure*} 
	
	The parameters~${\bm{\vartheta}^*_M = (1, 0.9)\transp}$ computed by AntLer for~$\nMC=200$ are employed to control the true system~\eqref{eq:ToyProbDiffEq} in~$\nMCtoyprob$ Monte Carlo runs. 
	Moreover, we compare the results to the Monte Carlo simulation using the optimal parameters obtained without anticipating learning~${\bm{\vartheta}^{0,*} = (0.9, 0.2)\transp}$. The respective results are shown in \Cref{fig:ToyPWithLearning} and \Cref{fig:ToyPWithoutLearning}.
	
	As shown in \Cref{fig:ToyPWithLearning}, the variance of the state is high for~$\bm{\vartheta}=\bm{\vartheta}^*_M$ at the beginning of the Monte Carlo simulation. This is due to the initially unknown system dynamics~$\gsc(\cdot)$. After approximately~$\step =80$, enough measurement data has been gathered to adequately track the reference trajectory. Despite differences in overall variance and learning time, the results agree qualitatively with the AntLer prediction shown in \Cref{fig:ToyProblemPredictionWithLearning}, which indicates that the kernel~\eqref{eq:SqExpKernel} was chosen adequately.
	
	For~${\bm{\vartheta}=\bm{\vartheta}^{0,*}}$, the variance of the state is very low throughout the simulation. This is because the parameters~${\bm{\vartheta}=\bm{\vartheta}^{0,*}}$ steer the system to a region of low model uncertainty. These results are in agreement with the predictions presented in \Cref{fig:ToyProblemPredictionWithoutLearning}.
	
The average cumulative cost for~${\bm{\vartheta}= \bm{\vartheta}^{0,*}}$ is~$859$. This is higher than for~${\bm{\vartheta}= \bm{\vartheta}_{\nMC}^{*}}$, which achieves an average cost of~$37$. This was expected from the AntLer predictions and \Cref{corollary:AmBetterthanwithoutAnticipating}.

	\section{Conclusion}
	\label{sect:Conclusion}
	\balance
	We have presented AntLer, a control design approach that anticipates the effect of online learning and optimizes data-independent parameters accordingly. By expressing model uncertainty with a Gaussian process model, we have formulated the parameter optimization problem as a stochastic optimal control problem, which AntLer solves approximately using sample average approximation. We have shown that AntLer approximates an optimal solution arbitrarily accurately with probability one for a sufficiently large number of samples. We have applied AntLer to a nonlinear system. The results have shown that model learning is correctly anticipated, which leads to a better choice of control parameters compared to the case where learning is not anticipated.
	
	In future work, we aim to apply AntLer to complex online learning-based control laws, such as learning-based model predictive control and online reinforcement learning.

	


	
	\bibliographystyle{IEEEtran}
	\bibliography{AllPhDReferences}

\begin{thebibliography}{10}
\providecommand{\url}[1]{#1}
\csname url@samestyle\endcsname
\providecommand{\newblock}{\relax}
\providecommand{\bibinfo}[2]{#2}
\providecommand{\BIBentrySTDinterwordspacing}{\spaceskip=0pt\relax}
\providecommand{\BIBentryALTinterwordstretchfactor}{4}
\providecommand{\BIBentryALTinterwordspacing}{\spaceskip=\fontdimen2\font plus
\BIBentryALTinterwordstretchfactor\fontdimen3\font minus
  \fontdimen4\font\relax}
\providecommand{\BIBforeignlanguage}[2]{{%
\expandafter\ifx\csname l@#1\endcsname\relax
\typeout{** WARNING: IEEEtran.bst: No hyphenation pattern has been}%
\typeout{** loaded for the language `#1'. Using the pattern for}%
\typeout{** the default language instead.}%
\else
\language=\csname l@#1\endcsname
\fi
#2}}
\providecommand{\BIBdecl}{\relax}
\BIBdecl

\bibitem{klenske2016gaussian}
E.~D. Klenske, M.~N. Zeilinger, B.~Sch{\"o}lkopf, and P.~Hennig, ``{G}aussian
  process-based predictive control for periodic error correction,'' \emph{IEEE
  Transactions on Control Systems Technology}, vol.~24, no.~1, pp. 110--121,
  2016.

\bibitem{murray2002nonlinear}
R.~Murray-Smith and D.~Sbarbaro, ``Nonlinear adaptive control using
  nonparametric {G}aussian process prior models,'' \emph{IFAC Proceedings
  Volumes}, vol.~35, no.~1, pp. 325--330, 2002.

\bibitem{kamthe2017data}
S.~Kamthe and M.~Deisenroth, ``Data-efficient reinforcement learning with
  probabilistic model predictive control,'' in \emph{International Conference
  on Artificial Intelligence and Statistics}, 2018, pp. 1701--1710.

\bibitem{koller2018learning}
T.~{Koller}, F.~{Berkenkamp}, M.~{Turchetta}, and A.~{Krause}, ``Learning-based
  model predictive control for safe exploration,'' in \emph{2018 IEEE
  Conference on Decision and Control (CDC)}, 2018, pp. 6059--6066.

\bibitem{umlauft2019feedback}
J.~{Umlauft} and S.~{Hirche}, ``Feedback linearization based on {G}aussian
  processes with event-triggered online learning,'' \emph{IEEE Transactions on
  Automatic Control}, pp. 1--1, 2019.

\bibitem{chowdhary2015bayesian}
G.~Chowdhary, H.~A. Kingravi, J.~P. How, P.~A. Vela \emph{et~al.}, ``{B}ayesian
  nonparametric adaptive control using {G}aussian processes.'' \emph{IEEE
  Trans. Neural Netw. Learning Syst.}, vol.~26, no.~3, pp. 537--550, 2015.

\bibitem{nguyen2011model}
D.~Nguyen-Tuong and J.~Peters, ``Model learning for robot control: a survey,''
  \emph{Cognitive processing}, vol.~12, no.~4, pp. 319--340, 2011.

\bibitem{bakker2006quasionline}
B.~{Bakker}, V.~{Zhumatiy}, G.~{Gruener}, and J.~{Schmidhuber}, ``Quasi-online
  reinforcement learning for robots,'' in \emph{2006 IEEE International
  Conference on Robotics and Automation}, May 2006, pp. 2997--3002.

\bibitem{astrom1994adaptive}
K.~J. Astrom and B.~Wittenmark, \emph{Adaptive Control}, 2nd~ed.\hskip 1em plus
  0.5em minus 0.4em\relax Boston, MA, USA: Addison-Wesley Longman Publishing
  Co., Inc., 1994.

\bibitem{Krstic1995}
M.~Krstic, I.~Kanellakopoulos, and P.~V. Kokotovic, \emph{Nonlinear and
  adaptive control design}.\hskip 1em plus 0.5em minus 0.4em\relax Wiley, 1995.

\bibitem{kocijan2016modelling}
J.~Kocijan, \emph{Modelling and control of dynamic systems using {G}aussian
  process models}.\hskip 1em plus 0.5em minus 0.4em\relax Springer, 2016.

\bibitem{berkenkamp2015safe}
F.~Berkenkamp and A.~P. Schoellig, ``Safe and robust learning control with
  {G}aussian processes,'' in \emph{2015 IEEE European Control Conference}, pp.
  2496--2501.

\bibitem{dayan1996exploration}
P.~Dayan and T.~J. Sejnowski, ``Exploration bonuses and dual control,''
  \emph{Machine Learning}, vol.~25, no.~1, pp. 5--22, 1996.

\bibitem{bar1974dual}
Y.~Bar-Shalom and E.~Tse, ``Dual effect, certainty equivalence, and separation
  in stochastic control,'' \emph{IEEE Transactions on Automatic Control},
  vol.~19, no.~5, pp. 494--500, 1974.

\bibitem{klenske2016dual}
E.~D. Klenske and P.~Hennig, ``Dual control for approximate {B}ayesian
  reinforcement learning.'' \emph{Journal of Machine Learning Research},
  vol.~17, no. 127, pp. 1--30, 2016.

\bibitem{kral2014gaussian}
L.~Kr{\'a}l, J.~Pr{\"u}her, and M.~{\v{S}}imandl, ``{G}aussian process based
  dual adaptive control of nonlinear stochastic systems,'' in \emph{22nd
  Mediterranean Conference on Control and Automation}.\hskip 1em plus 0.5em
  minus 0.4em\relax IEEE, 2014, pp. 1074--1079.

\bibitem{Berkenkamp2016}
F.~Berkenkamp, R.~Moriconi, A.~P. Schoellig, and A.~Krause, ``Safe learning of
  regions of attraction for uncertain, nonlinear systems with {G}aussian
  processes,'' in \emph{2016 IEEE Conference on Decision and Control}, pp.
  4661--4666.

\bibitem{berkenkamp2016safe}
F.~Berkenkamp, A.~P. Schoellig, and A.~Krause, ``Safe controller optimization
  for quadrotors with {G}aussian processes,'' in \emph{2016 IEEE International
  Conference on Robotics and Automation}, pp. 491--496.

\bibitem{neumann2019data}
M.~Neumann-Brosig, A.~Marco, D.~Schwarzmann, and S.~Trimpe, ``Data-efficient
  autotuning with {B}ayesian optimization: An industrial control study,''
  \emph{IEEE Transactions on Control Systems Technology}, 2019.

\bibitem{Rasmussen2006}
C.~E. Rasmussen and C.~K. Williams, ``{G}aussian processes for machine
  learning. 2006,'' \emph{The MIT Press, Cambridge, MA, USA}, 2006.

\bibitem{deisenroth2015gaussian}
M.~P. Deisenroth, D.~Fox, and C.~E. Rasmussen, ``{G}aussian processes for
  data-efficient learning in robotics and control,'' \emph{IEEE Transactions on
  Pattern Analysis and Machine Intelligence}, vol.~37, no.~2, pp. 408--423,
  2015.

\bibitem{beckers2018misspecified}
T.~{Beckers}, J.~{Umlauft}, and S.~{Hirche}, ``Mean square prediction error of
  misspecified gaussian process models,'' in \emph{2018 IEEE Conference on
  Decision and Control}, Dec 2018, pp. 1162--1167.

\bibitem{armin2020confidence}
A.~{Lederer}, H.~{Qing}, and S.~{Hirche}, ``Confidence regions for simulations
  with learned probabilistic models,'' \emph{American Control Conference}, pp.
  1--1, 2020.

\bibitem{shapiro2009lectures}
A.~Shapiro, D.~Dentcheva, and A.~Ruszczy{\'n}ski, \emph{Lectures on stochastic
  programming: modeling and theory}.\hskip 1em plus 0.5em minus 0.4em\relax
  SIAM, 2009.

\bibitem{micchelli2006universal}
C.~A. Micchelli, Y.~Xu, and H.~Zhang, ``Universal kernels,'' \emph{Journal of
  Machine Learning Research}, vol.~7, no. Dec, pp. 2651--2667, 2006.

\end{thebibliography}

\end{document}